\documentclass{article} %
\usepackage{iclr2027_conference,times}

\usepackage{amsmath,amsfonts,bm}

\usepackage{hyperref}
\usepackage{url}

\usepackage{graphicx}
\usepackage{xcolor}
\usepackage{colortbl}
\definecolor{catgray}{gray}{0.92}
\usepackage{booktabs}
\usepackage{multirow}
\usepackage{amsmath}
\usepackage{amssymb}
\usepackage{amsthm}
\usepackage{enumitem}
\usepackage{tcolorbox}
\tcbuselibrary{breakable}
\DeclareUnicodeCharacter{2011}{\mbox{-}}
\newtcolorbox{promptbox}{
  breakable,
  colback=white,
  colframe=black,
  boxrule=0.6pt,
  arc=3mm,
  left=8pt,
  right=8pt,
  top=8pt,
  bottom=8pt,
  before skip=8pt,
  after skip=12pt
}

\AddToHook{env/table/begin}{%
  \setlength{\belowcaptionskip}{6pt}%
}

\theoremstyle{plain}
\newtheorem{proposition}{Proposition}

\title{RepForcing: Representation-Space Initialization\\
for Autoregressive Video Distillation}

\author{%
\makebox[\textwidth][c]{%
\textbf{Siming Fu}$^{1}$\thanks{Equal Contribution.}\hspace{0.8em}%
\textbf{Zheming Fu}$^{1}$\footnotemark[1]\hspace{0.8em}%
\textbf{Lianrui Mu}$^{1}$\hspace{0.8em}%
\textbf{Haoyang Huang}$^{1}$\hspace{0.8em}%
\textbf{Chaojie Gu}$^{2}$\hspace{0.8em}%
\textbf{Ruizhe He}$^{2,*}$\thanks{Corresponding Author.}%
}\\[0.5em]
{\small $^{1}$Joy Future Academy \quad $^{2}$Zhejiang University}%
}

\makeatletter
\renewcommand{\@maketitle}{%
  \vbox{\hsize\textwidth
    \centering
    {\normalfont\Large\bfseries\@title\par}%
    \vskip 0.8em
    {\normalfont\normalsize\@author\par}%
    \vskip 0.3in minus 0.1in
  }%
}
\makeatother

\iclrfinalcopy
\begin{document}

\maketitle
\lhead{}

\begin{figure}[ht]
\vspace{-16pt}
\begin{center}
	\includegraphics[width=0.97\linewidth]{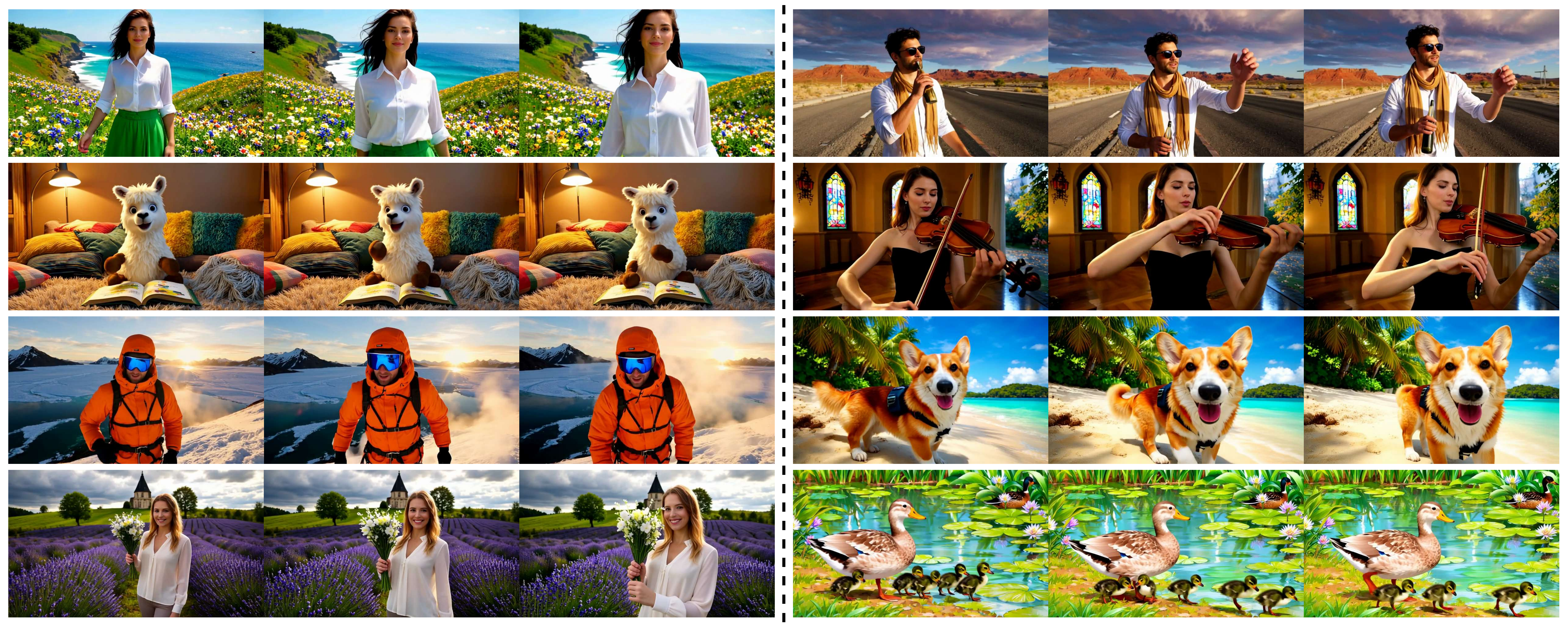}
\end{center}
\caption{\textbf{High-Quality Videos Generated by RepForcing.}
The examples span human activities, animal behavior, and stylized scenes,
showing details such as clothing folds, fur, and vegetation. Across the
displayed frames, subjects change pose and position while retaining their
appearance and surrounding scene structure.}
\label{teaser}
\end{figure}

\begin{abstract}
Distilling pretrained video diffusion models into few-step autoregressive (AR)
generators typically combines student initialization with distribution matching
distillation (DMD). An intermediate AR teacher can provide causally compatible
targets for ODE initialization.
However, causally compatible targets alone do not ensure high-quality generation
under a short sampling schedule. We analyze the posterior-mean predictions
favored by direct latent regression and empirically find that causal ODE
initialization alleviates blur but still leaves fine detail unresolved.
To address this limitation, we propose \textbf{\textit{Representation Forcing}},
which learns few-step causal initialization through representation-space
supervision rather than intermediate-teacher trajectory regression.
The student is trained with teacher forcing by comparing clean predictions
and training targets through normalized features of a frozen video diffusion
transformer. This phase reuses the synthetic videos generated by a bidirectional
model for AR teacher training and replaces both teacher training and causal ODE
distillation, without generating additional causal ODE trajectories. The subsequent asymmetric DMD procedure remains unchanged.
Empirical results show that, at four sampling steps, RepForcing better preserves
generation quality than intermediate-teacher-based initialization before DMD.
After DMD, our four-step chunk-wise model
improves VBench Total from \textbf{\textit{84.04 to 84.66 (+0.62 points)}} and
VisionReward from \textbf{\textit{6.326 to 10.074 (+59.25\%)}} over the reported
causal-ODE-based baseline. Code at \href{https://github.com/LMIUD/RepForcing}{RepForcing}.
\end{abstract}

\section{Introduction}
\label{sec:intro}

Autoregressive (AR) video diffusion enables streaming generation by producing
video blocks conditioned on previously generated content. This formulation
supports interactive applications including world modeling, game simulation,
and embodied intelligence~\citep{jin2025pyramidal,teng2025magi,chen2025skyreels,genie3}.
Despite these capabilities, repeatedly denoising each block increases
inference cost. Reducing the number of sampling steps while preserving visual
quality is therefore important for real-time video generation.

To reduce this cost, distillation methods~\citep{yin2025causvid,huang2025self}
convert pretrained bidirectional video diffusion models into few-step AR
students. The training pipeline typically consists of an initialization stage
followed by asymmetric distribution matching distillation
(DMD)~\citep{yin2024one,yin2024improved}. Beyond accelerating sampling,
initialization must account for the difference between bidirectional teachers
and causal students. Prior analysis~\citep{zhu2026causal} identifies frame-level
injectivity as a requirement for recovering the teacher's flow map through ODE
distillation: each input available to the student must determine a unique
endpoint. A bidirectional teacher can violate this requirement through its
dependence on future frames unavailable to the student. Training an intermediate
AR teacher and distilling its causal trajectories provides compatible targets,
but requires two training stages before DMD.

However, resolving causal compatibility does not by itself ensure high-quality
few-step generation. In this work, we revisit the supervision used to construct
the initialization. Direct latent regression favors posterior-mean clean
predictions: when the noisy input and its conditions admit different plausible
continuations, coordinate-wise averaging can blur texture and blend incompatible
structures. This concerns individual clean predictions rather than the output
distribution of a complete multi-step sampler. Empirically, directly shortening
the sampling schedule of the teacher-forced AR diffusion model produces blurred
outputs. Causal ODE initialization alleviates this blur, but still leaves fine
detail unresolved under few-step sampling. Figure~\ref{fig:nfe_sweep} compares
the initialization methods as the number of sampling steps decreases, revealing
a larger quality gap between ODE initialization and RepForcing at four steps.
Thus, deterministic teacher targets do not remove the observed limitation. These
findings motivate an alternative to intermediate-teacher initialization, rather
than an additional distillation stage.

To address this limitation, we propose \emph{Representation Forcing}
(RepForcing), which learns few-step causal initialization directly in a
pretrained video representation space. The student predicts clean video blocks
under teacher forcing, and its predicted video is compared with the training
target through normalized features of a frozen video diffusion transformer.
Unlike coordinate-wise latent regression, this objective need not favor the
latent posterior mean. RepForcing preserves the causal student architecture
and reuses the synthetic videos generated by a bidirectional Wan model from
VidProM~\citep{wang2024vidprom} prompts for intermediate AR teacher training.
It requires no additional causal ODE trajectory pairs. A single teacher-forced
phase replaces both intermediate-teacher training and ODE distillation, after
which the student undergoes the same asymmetric DMD refinement using its own
generated prefixes.

Experiments evaluate both the resulting initialization and the final generator.
Before DMD, RepForcing better preserves generation quality than
intermediate-teacher-based initialization at four sampling steps.
After DMD, it improves VBench Total and
VisionReward over the reported intermediate-teacher baselines in both chunk-wise
and frame-wise generation using the same number of sampling steps
(Table~\ref{tab:main_comparison}). For four-step chunk-wise generation, the
improvements are 0.62 points and 59.25\%, respectively.
\textbf{Our contributions are summarized as follows:}
\begin{itemize}
\item We analyze the limitations of intermediate-teacher initialization
  for few-step autoregressive generation. Our analysis connects the
  posterior-mean behavior of latent regression to the loss of detail
  under a short sampling schedule, while controlled comparisons show
  that causal ODE initialization alleviates blur but leaves fine detail
  unresolved.
\item We propose \textbf{\textit{Representation Forcing}}, which reformulates
  few-step causal initialization as regression in a pretrained video
  representation space. It replaces intermediate-teacher training and causal
  ODE distillation with a single teacher-forced phase on the same synthetic
  training videos.
\item Experimental results show that RepForcing better preserves generation
  quality than intermediate-teacher-based initialization at four sampling
  steps before DMD. After DMD, it improves VBench Total and VisionReward
  over the reported Causal Forcing baselines.
\end{itemize}

\section{Related Work}
\label{sec:related}

\paragraph{Autoregressive Video Generation.}
AR video diffusion produces a video block by block and runs diffusion within
each block, which enables streaming
generation~\citep{jin2025pyramidal,teng2025magi,chen2025skyreels}. During
training, each block is conditioned on clean history under teacher forcing or on
noisy history under diffusion forcing~\citep{chen2024diffusion}, whereas at
inference it is conditioned on generated history. Self
Forcing~\citep{huang2025self} closes this gap by training on self-rollouts with
KV caching, AAPT~\citep{lin2026autoregressive} uses adversarial post-training for
single-step streaming, and further work extends the generation
horizon~\citep{cui2026selfforcingpp,liu2025rolling,longlive}.

\paragraph{Few-Step Initialization for AR Distillation.}
CausVid~\citep{yin2025causvid} and Self Forcing initialize the AR student with
ODE pairs from a bidirectional teacher before DMD~\citep{yin2024one,yin2024improved}.
Causal Forcing~\citep{zhu2026causal} replaces this teacher with an AR model
trained under teacher forcing, so that the distilled trajectories respect
causality. Causal Forcing++~\citep{zhao2026causalforcingpp} and
Causal-rCM~\citep{zheng2026causalrcm} substitute causal consistency distillation
or training for the trajectory stage, building on consistency
models~\citep{song2023consistency,lu2025scm}. All of these learn the few-step
initialization from an intermediate multi-step teacher; we instead obtain it
directly from the shared bidirectional-model-generated training videos by
changing the regression metric.

\paragraph{Learned Representations as Training Signals.}
Perceptual losses compare predictions and targets through a pretrained
network~\citep{simonyan2014very,johnson2016perceptual}, and
LPIPS~\citep{zhang2018unreasonable} calibrates such distances to human
judgments. In generative modeling, REPA~\citep{yu2024representation} and
Geometry Forcing~\citep{wu2026geometryforcing} align a generator's internal activations
with an external encoder such as DINOv2~\citep{oquab2023dinov2},
REPA-E~\citep{leng2025repa} extends this alignment to end-to-end VAE tuning,
representation autoencoders~\citep{zheng2025rae} perform diffusion in the space
of a pretrained encoder, and LADD~\citep{sauer2024fast} uses diffusion features
for adversarial distillation. Our objective differs from each of them: it preserves
the generator architecture and VAE, requires no discriminator, and
uses a frozen video DiT only to measure the distance between a predicted
clean latent and its target.

\section{Preliminaries}
\label{sec:prelim}

Following Causal Forcing~\citep{zhu2026causal}, let
$(X_0,y)\sim p_{\mathrm{data}}$ be a video VAE latent~\citep{rombach2022high,blattmann2023align}
and text condition, with $X_0=(\mathbf x_0^1,\ldots,\mathbf x_0^N)$ split into
$N$ frame- or chunk-wise blocks. We reuse its first-stage dataset
$\mathcal D_{\mathrm{Bi}}$: approximately 3K videos synthesized by bidirectional
Wan from VidProM~\citep{wang2024vidprom} prompts. ``Clean'' and ``ground truth''
refer to these unnoised synthetic targets.

\subsection{Flow matching}
\label{sec:prelim_flow}

For a clean latent $\mathbf x_0$, the forward noising kernel $q_{t|0}$ is
\begin{equation}
\label{eq:prelim_path}
\mathbf x_t=\alpha_t\mathbf x_0+\sigma_t\boldsymbol\epsilon,
\qquad \boldsymbol\epsilon\sim\mathcal N(\mathbf0,I),
\qquad \alpha_t=1-t,\quad \sigma_t=t,
\end{equation}
where $t\in[0,1]$ is normalized noise time, $I$ is the identity covariance,
and noise is independent of the data. Given condition $h$, flow
matching~\citep{liu2022flow,lipmanflow} trains velocity network $v_\theta$
and defines its clean prediction $G_\theta$ by
\begin{equation}
\label{eq:prelim_fm}
\begin{gathered}
\mathcal L_{\mathrm{FM}}(\theta)
=\mathbb E_{\mathbf x_0,h,t,\boldsymbol\epsilon}
\left[w(t)\left\|v_\theta(\mathbf x_t,t,h)
-(\boldsymbol\epsilon-\mathbf x_0)\right\|_2^2\right],\\[4pt]
G_\theta(\mathbf x_t,t,h)=\mathbf x_t-t\,v_\theta(\mathbf x_t,t,h).
\end{gathered}
\end{equation}
Here $\theta$ denotes trainable parameters and $w(t)>0$ weights the sampled
times. Sampling integrates the probability flow ODE
$\mathrm d\mathbf z_t=v_\theta(\mathbf z_t,t,h)\,\mathrm dt$ from standard
Gaussian $\mathbf z_1$ to $t=0$~\citep{song2021score}. Unlike directly noised
$\mathbf x_t$, $\mathbf z_t$ lies on a sampling trajectory; its integrated
endpoint need not equal the single prediction $G_\theta$.

\begin{figure}[t]
\centering
\includegraphics[width=\linewidth]{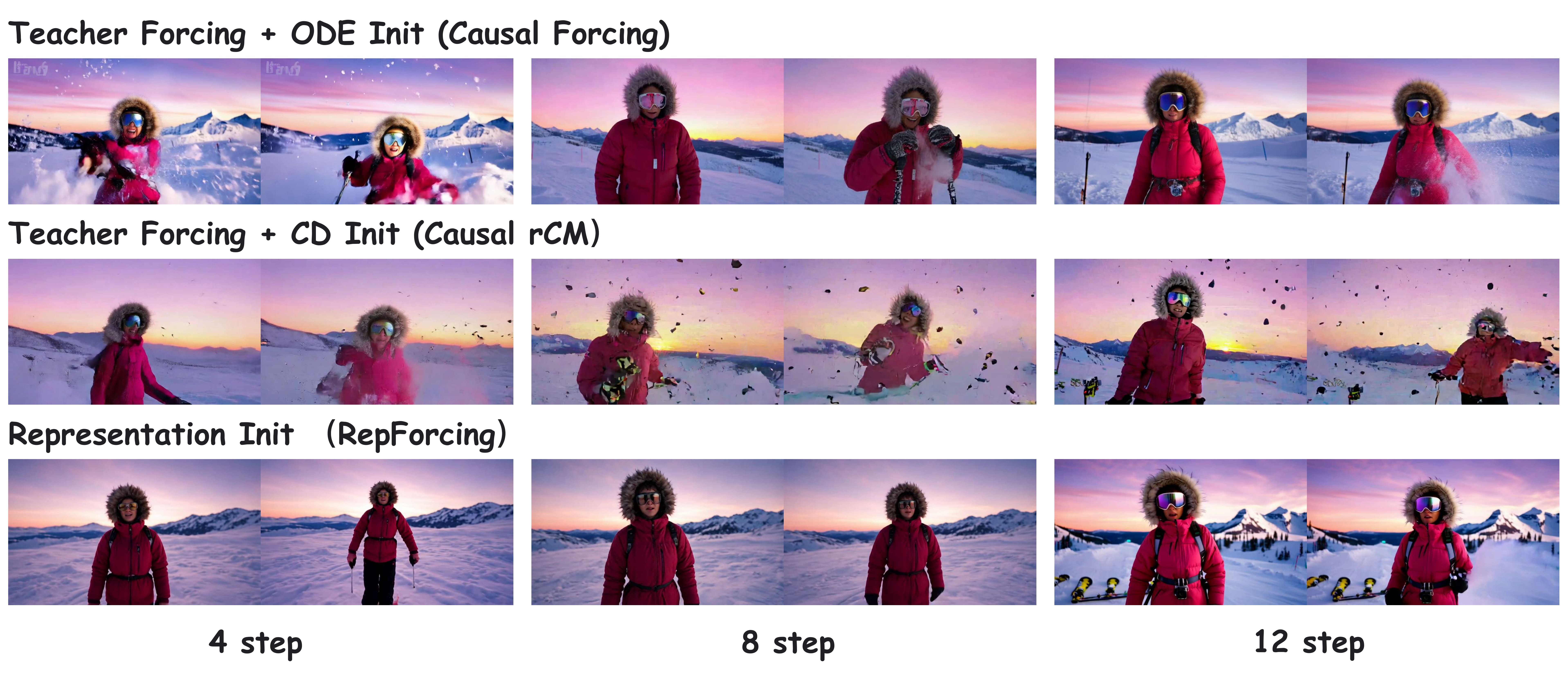}
\caption{\textbf{RepForcing Initialization Better Preserves Generation Quality as Sampling Steps Decrease.}
Rows compare three initialization strategies: teacher-forced AR training
followed by ODE distillation, teacher-forced AR training followed by consistency
distillation, and direct representation-space initialization, from top to
bottom. All results are shown before subsequent refinement. Columns show 4, 8, and 12 sampling steps from
left to right, with two frames per setting for the same prompt.
With more denoising steps, ODE initialization approaches the visual quality
of RepForcing. At four steps, however, the intermediate-teacher initializations
exhibit blur or structural artifacts, while RepForcing maintains coherent
subject and scene structure.}
\label{fig:nfe_sweep}
\end{figure}

\subsection{Autoregressive diffusion and ODE initialization}
\label{sec:prelim_ar}

With prefix $\mathbf x_0^{<i}$ (empty for $i=1$) and condition
$h_i=(\mathbf x_0^{<i},y)$, the model factorizes
\begin{equation}
\label{eq:prelim_ar}
p_\theta(X_0\mid y)=\prod_{i=1}^{N}p_\theta(\mathbf x_0^i\mid h_i).
\end{equation}
Teacher forcing uses clean prefixes and independently noised blocks from
Eq.~(\ref{eq:prelim_path}), with times $\mathbf t=(t_1,\ldots,t_N)$ and
noises $E=(\boldsymbol\epsilon^1,\ldots,\boldsymbol\epsilon^N)$.
Set $U^i=(\mathbf x_{t_i}^i,t_i,h_i)$ and
$\hat{\mathbf x}_0^i=G_\theta(U^i):=G_\theta(\mathbf x_{t_i}^i,t_i,h_i)$,
and collect predictions as $\hat X_0=(\hat{\mathbf x}_0^1,\ldots,\hat{\mathbf x}_0^N)$.
A block-causal mask on concatenated clean and noisy sequences enables one-pass
training~\citep{teng2025magi,zhu2026causal}; free AR sampling instead uses
generated prefixes and produces $\tilde X_0$.

Causal Forcing trains an AR teacher $v_\psi$ on $\mathcal D_{\mathrm{Bi}}$,
then freezes $\psi$ and samples trajectories under $h_i$ at the few-step
schedule $\mathcal S\subset(0,1]$. The resulting tuple distribution
$\mathcal D_{\mathrm{ODE}}$ supplies endpoint regression pairs:
\begin{equation}
\label{eq:cf_stage2}
\mathcal L_{\mathrm{ODE}}(\theta)
=\mathbb E_{(i,h_i,t,\mathbf z_t^i,\mathbf z_0^i)\sim\mathcal D_{\mathrm{ODE}}}
\left[\left\|G_\theta(\mathbf z_t^i,t,h_i)-\mathbf z_0^i\right\|_2^2\right].
\end{equation}
Here $\mathbf z_0^i$ is the teacher-generated endpoint for the causal input
$(\mathbf z_t^i,t,h_i)$, not the training clip's target; only the prefix comes
from that clip. A well-defined teacher ODE makes this endpoint deterministic.

\section{Method}
\label{sec:method}

\begin{figure}[t]
\centering
\includegraphics[width=\linewidth,trim=0bp 0bp 302bp 0bp,clip]{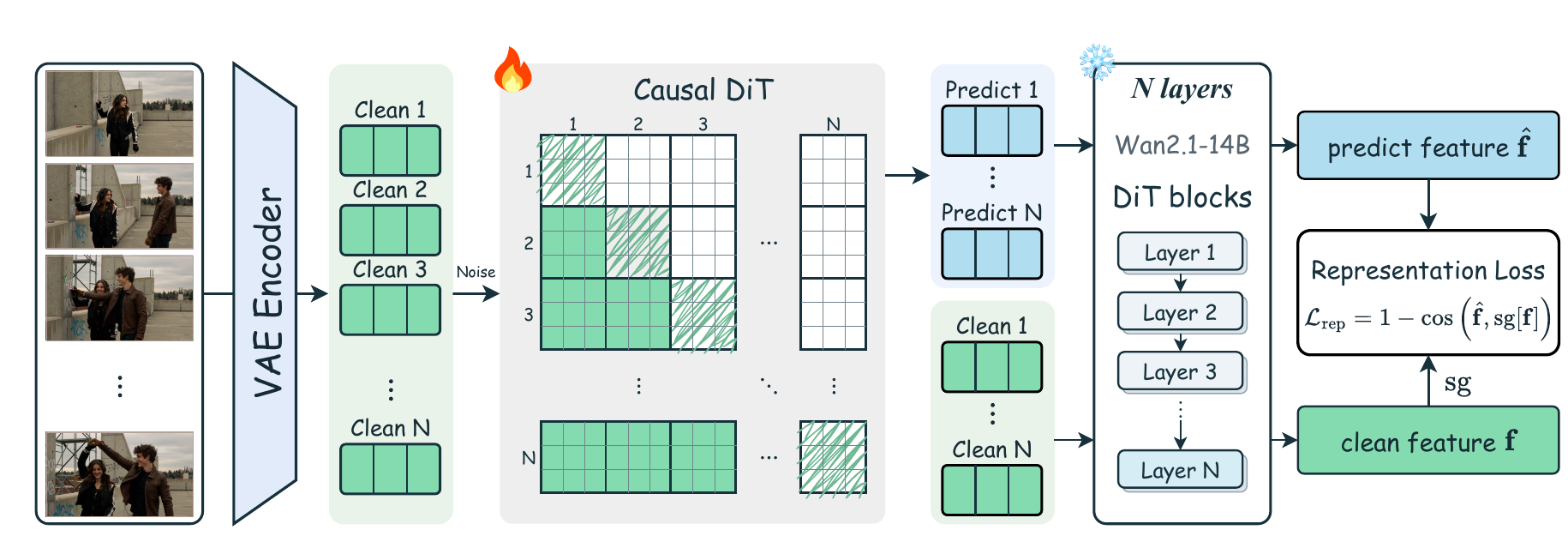}
\caption{\textbf{Representation Forcing Initialization Pipeline.}
Training videos are encoded into clean latent blocks by a VAE encoder.
The trainable causal DiT reconstructs clean blocks from noisy inputs,
conditioned on preceding clean blocks under teacher forcing.
The predicted and target video latents are passed separately through the same
truncated, frozen Wan2.1-14B to obtain features $\hat{\mathbf f}$ and
$\mathbf f$. A token-wise cosine loss compares these features, with
stop-gradient ($\operatorname{sg}$) on the target branch. Gradients propagate
through the frozen extractor to update the causal DiT.  This initialization
replaces intermediate AR teacher training and causal ODE distillation.}
\label{fig:training_routes}
\end{figure}

\subsection{Analysis: Suboptimality of Initialization with an Intermediate Teacher}
\label{sec:analysis}

We analyze the limitations of direct latent regression and examine the
remaining deficiencies of causal ODE initialization under few-step sampling,
following the training setup described in Sec.~\ref{sec:prelim_ar}.

\paragraph{Latent $\ell_2$ Regression Blurs Few-Step Predictions.}
To analyze the direct predictions of the intermediate teacher, write the
clean-block regression loss using the quantities defined in
Sec.~\ref{sec:prelim_ar}:
\begin{equation}
\label{eq:latent_mse}
\mathcal L_{\mathrm{MSE}}(\theta)
=\mathbb E_{X_0,y,E,\mathbf t}
\left[\sum_{i=1}^N\left\|\hat{\mathbf x}_0^i-\mathbf x_0^i\right\|_2^2\right].
\end{equation}
By the parameterization in Eq.~(\ref{eq:prelim_fm}), each clean-block
squared error is $t_i^2$ times its velocity squared error. Thus the losses
differ in timestep weighting, but have the same unrestricted conditional
optimum at a fixed positive time.

\begin{proposition}[Posterior mean as the minimizer of latent $\ell_2$ regression]
\label{prop:collapse}
Fix a block $i$ and $t\in(0,1]$ under the independent Gaussian noising
process of Sec.~\ref{sec:prelim_ar}. Assume
$\mathbb E\|\mathbf x_0^i\|_2^2<\infty$ and that
$p_{\mathrm{data}}(\mathbf x_0^i\mid h_i)$ is not a point mass on a set of
conditions $h_i$ of positive probability. Over square-integrable measurable
predictors, the fixed-time blockwise risk in Eq.~(\ref{eq:latent_mse}) has
the almost-surely unique minimizer
\begin{equation}
\label{eq:collapse}
\hat{\mathbf{x}}_0^{i,*}
=\mathbb E[\mathbf x_0^i\mid\mathbf x_t^i,h_i],
\qquad
\mathbb{E}\big\|\hat{\mathbf{x}}_0^{i,*}\big\|_2^2<\mathbb{E}\big\|\mathbf{x}_0^i\big\|_2^2,
\end{equation}
so the optimal predictions are not distributed as the data. Proof in
Appendix~\ref{app:proof_collapse}.
\end{proposition}

Proposition~\ref{prop:collapse} shows that latent regression favors the
conditional mean of clean continuations compatible with the noisy input and
its conditions. When these continuations differ in texture placement or motion
phase, coordinate-wise averaging can blend incompatible structures. Multi-step
sampling refines successive predictions as noise decreases, whereas a short
schedule provides fewer refinement steps. The proposition characterizes direct
clean predictions, rather than the output distribution of the complete sampler.
Empirically, the teacher-forced AR diffusion model produces blurred outputs
under four-step sampling, indicating that diffusion training alone does not
provide the required few-step initialization in this setting.

\paragraph{Causal ODE Initialization Alleviates Blur but Leaves Fine Detail Unresolved.}
Causal ODE initialization replaces the data-regression target with a
deterministic endpoint generated by the AR teacher, as defined in
Eq.~(\ref{eq:cf_stage2}). This removes the conditional target ambiguity of
direct latent regression. However, target determinism alone does not ensure
high-quality generation in the few-step regime. As shown in
Fig.~\ref{fig:nfe_sweep}, ODE initialization produces coherent results with
eight or twelve sampling steps, but exhibits degraded visual quality when
sampling is reduced to four steps. This
observation motivates replacing the intermediate-teacher initialization route
with direct representation-space supervision.

\subsection{Representation Forcing}
\label{sec:repforcing}

Building on the above analysis, we propose Representation Forcing (RepForcing),
a two-stage method consisting of representation-space teacher forcing
(Fig.~\ref{fig:training_routes}) and asymmetric DMD. The first stage directly
trains the causal student on the same $\mathcal D_{\mathrm{Bi}}$ used for
Causal Forcing's AR teacher training. These synthetic videos supply the clean
targets and prefixes, without intermediate AR teacher training or additional
causal ODE trajectory generation.

\paragraph{Representation-Space Supervision Preserves Fine Detail under Few-Step Sampling.}
Instead of regressing teacher-generated trajectories, we compare the student's
clean predictions with the training videos in a pretrained representation
space. Minimizing normalized-feature discrepancy need not favor the
coordinate-wise latent mean (Appendix~\ref{app:feature_regression}). If the
representation distinguishes plausible continuations from their blurred
blends, this objective can disfavor those blends. Empirically,
Fig.~\ref{fig:nfe_sweep} shows that RepForcing better preserves generation
quality at four steps than intermediate-teacher-based initialization, before
subsequent refinement.

\paragraph{Stage 1: Teacher Forcing in Representation Space.}
\label{sec:method_loss}
Given the teacher-forced predictions $\hat X_0$ defined in
Sec.~\ref{sec:prelim_ar} and the target video $X_0$, a shared frozen video DiT
$\Phi_{L,\tau}$ extracts their token features. Here $L$ is the one-based
extraction block index, $\tau$ is the extractor's timestep embedding value, and $K$
is the number of tokens. Feature extraction uses native bidirectional
attention without changing the student's causal architecture:
\begin{equation}
\label{eq:method_features}
\hat{\mathbf f}_k=[\Phi_{L,\tau}(\hat X_0;y)]_k,
\qquad \mathbf f_k=[\Phi_{L,\tau}(X_0;y)]_k,
\qquad k=1,\ldots,K.
\end{equation}
The subscript $k$ selects a token feature vector. For a video latent $X$,
$\bar\Phi(X;y)$ denotes the concatenated unit-normalized token features of
$\Phi_{L,\tau}(X;y)$, with $L$ and $\tau$ fixed.
We optimize the token-wise cosine objective
\begin{equation}
\label{eq:feat_loss}
\mathcal L_{\mathrm{rep}}(\theta)
=\mathbb E_{X_0,y,E,\mathbf t}
\left[\frac1K\sum_{k=1}^{K}
\left(1-\cos\!\left(\hat{\mathbf f}_k,
\operatorname{sg}[\mathbf f_k]\right)\right)\right],
\end{equation}
where $\cos$ denotes cosine similarity between nonzero feature vectors and
$\operatorname{sg}$ denotes stop-gradient. The target features are detached,
while gradients through the frozen extractor update the causal student.
No latent MSE or adversarial term is added. The per-video loss is equivalently
$\|\bar\Phi(\hat X_0;y)-\bar\Phi(X_0;y)\|_2^2/(2K)$, linking the cosine
objective to regression on normalized features.

\paragraph{Criteria for Representation Selection.}
\label{sec:method_extractor}
The representation loss should favor detailed predictions rather than
latent blends. We use two criteria for the extractor $\Phi_{L,\tau}$.
\textbf{\textit{First}}, we favor video representations that capture motion
and inter-frame relations, which frame-wise image features do not explicitly
model. This motivates extracting features from whole video sequences rather
than individual frames. \textbf{\textit{Second}}, features
should be read from an intermediate depth. We expect shallow features to be close to an isometry of
the latent space and thus to inherit its averaging behaviour, and the deepest
features to become invariant to fine appearance.

\paragraph{Stage 2: Asymmetric DMD.}
\label{sec:method_pipeline}
Starting from the representation-space initialization, we apply asymmetric
DMD following Self Forcing~\citep{huang2025self}, as in Causal
Forcing~\citep{zhu2026causal}. The student performs few-step autoregressive
rollouts conditioned on its own generated prefixes, reducing the mismatch
between teacher-forced training and inference. We retain this refinement
stage unchanged; the feature extractor is no longer used.

\section{Experiments}
\label{sec:exp}

\begin{figure}[t]
\centering
\includegraphics[width=\linewidth]{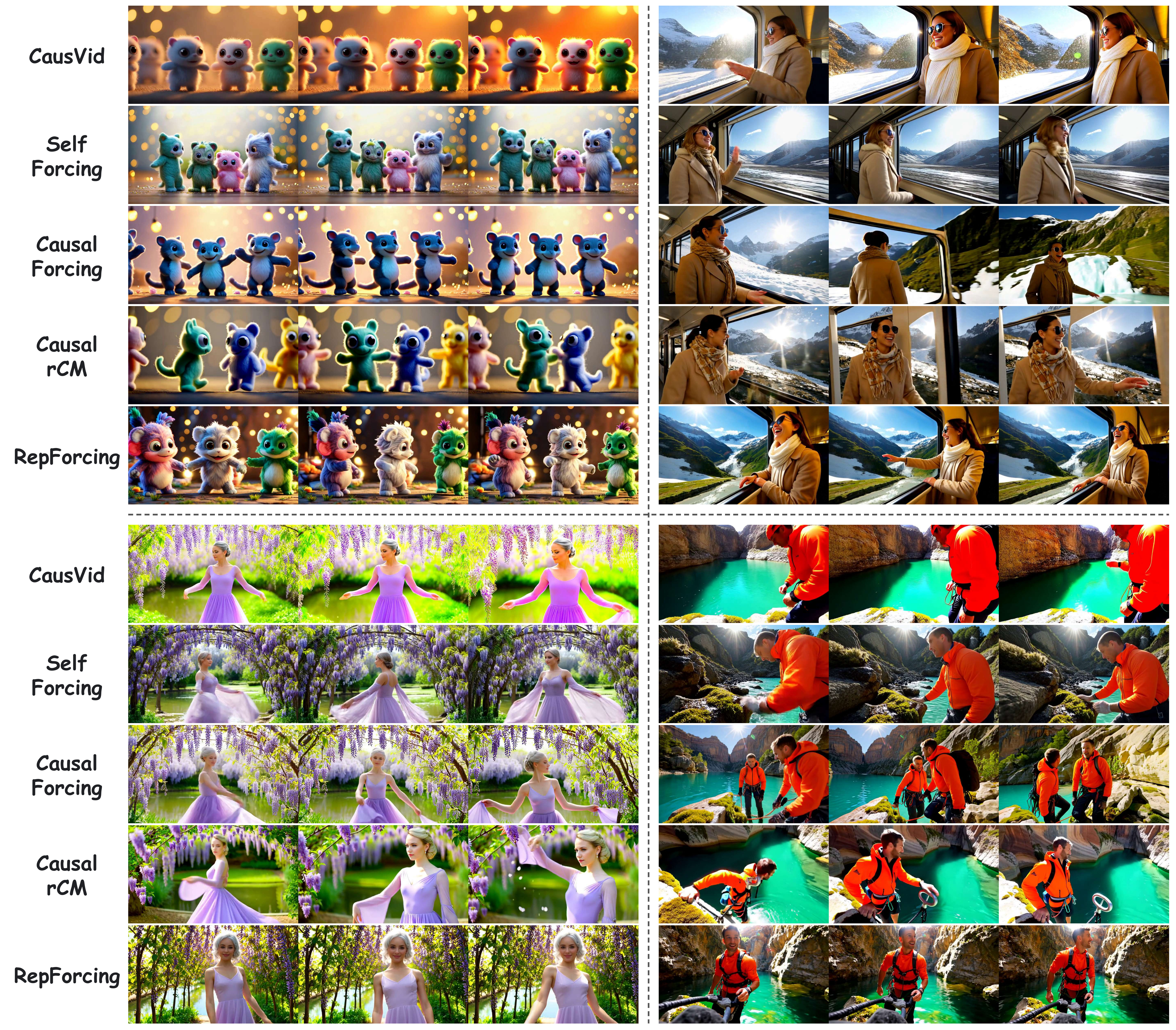}
\caption{\textbf{Qualitative comparison of chunk-wise distilled autoregressive
models.} For each of four prompts, three sampled frames are shown for CausVid,
Self Forcing, Causal Forcing, Causal-rCM, and RepForcing. The RepForcing examples
retain visible fur strands and facial detail in the character scene, together
with clothing folds, hanging blossoms, and rock textures in the human scenes.
Changes in body and hand pose are visible across the sampled frames.}
\label{fig:baseline}
\end{figure}

We evaluate whether the initial teacher-forced training can replace an independent
few-step distillation stage. We report the quality of the initialization before DMD, the quality of the
final model after DMD, and results under several sampling settings. Published
reference numbers and our matched initialization experiments are reported separately.
Additional results are provided in the appendix.

\begin{table}[t]
\centering
\small
\setlength{\tabcolsep}{4pt}
\caption{Comparison with existing methods. Dynamic., Vision., and Instruct. denote
Dynamic Degree, VisionReward, and Instruction Following. \textbf{Bold} marks
column-wise best results within each distilled AR group, including ties;
empty cells are unreported. Baseline scores follow \citet{zhu2026causal};
frame-wise four-step Causal Forcing++ scores come from its own
report~\citep{zhao2026causalforcingpp}, and Causal-rCM scores are obtained from
evaluations of pretrained weights~\citep{zheng2026causalrcm}.}
\label{tab:main_comparison}
\resizebox{\textwidth}{!}{%
\begin{tabular}{lcccccc}
\toprule
Model & Total $\uparrow$ & Quality $\uparrow$ & Semantic $\uparrow$ & Dynamic. $\uparrow$ & Vision. $\uparrow$ & Instruct. $\uparrow$ \\
\midrule
\rowcolor{catgray}
\multicolumn{7}{l}{\textit{Bidirectional video models}} \\
LTX-1.9B~\citep{hacohen2024ltx} & 79.83 & 81.88 & 71.62 & 46 & $-6.218$ & $-38$ \\
Wan2.1-1.3B~\citep{wan2025wan} & 83.37 & 84.30 & 79.65 & 61 & 5.275 & 42 \\
\midrule
\rowcolor{catgray}
\multicolumn{7}{l}{\textit{Autoregressive video models}} \\
NOVA~\citep{deng2024nova} & 80.31 & 80.66 & 78.92 & 46 & $-7.381$ & $-16$ \\
Pyramid Flow~\citep{jin2025pyramidal} & 80.75 & 83.41 & 70.11 & 16 & 4.055 & $-2$ \\
SkyReels-V2-1.3B~\citep{chen2025skyreels} & 81.97 & 83.96 & 74.01 & 37 & 3.584 & 32 \\
MAGI-1-4.5B~\citep{teng2025magi} & 78.88 & 81.67 & 67.72 & 42 & 0.773 & 8 \\
\midrule
\rowcolor{catgray}
\multicolumn{7}{l}{\textit{Distilled autoregressive models: chunk-wise}} \\
CausVid~\citep{yin2025causvid} & 81.33 & 83.98 & 70.72 & 62 & 5.741 & 12 \\
Self Forcing~\citep{huang2025self} & 83.74 & 84.48 & 80.77 & 57 & 5.820 & 48 \\
Causal Forcing~\citep{zhu2026causal} & 84.04 & 84.59 & \textbf{81.84} & 68 & 6.326 & 56 \\
Causal-rCM~\citep{zheng2026causalrcm} & 84.37 & 85.02 & 81.73 & \textbf{71} & 8.211 & 64 \\
Ours & \textbf{84.66} & \textbf{85.55} & 81.10 & 69 & \textbf{10.074} & \textbf{78} \\
\midrule
\rowcolor{catgray}
\multicolumn{7}{l}{\textit{Distilled autoregressive models: frame-wise}} \\
Causal Forcing~\citep{zhu2026causal} & 83.75 & 84.35 & \textbf{81.37} & 64 & 6.204 & 42 \\
Causal Forcing++~\citep{zhao2026causalforcingpp} & 84.10 & 84.94 & 80.75 & 71 & 6.798 & 47 \\
Ours & \textbf{84.38} & \textbf{85.67} & 79.21 & \textbf{78} & \textbf{10.598} & \textbf{80} \\
\bottomrule
\end{tabular}%
}
\vspace{-2em}
\end{table}

\subsection{Setup}
\label{sec:exp_setup}

\paragraph{Implementation Details.}
RepForcing's representation-space initialization uses the shared synthetic
dataset $\mathcal D_{\mathrm{Bi}}$ defined in Sec.~\ref{sec:prelim}.
We use a Wan2.1-1.3B causal student, a frozen Wan2.1-14B feature extractor and
DMD teacher, and a Wan2.1-1.3B auxiliary score model. Final models are evaluated
with four-step generation in both chunk-wise and frame-wise settings, using
three latent frames per chunk in the former. Each model is refined with its
corresponding four-step DMD rollout schedule. Representation-model and
extraction-depth ablations use the chunk-wise four-step setting, with the same
training data, initial student weights, Stage~1 update budget, and DMD teacher,
update schedule, and update budget.
For feature supervision, we read block $L=8$ (one-based) of the frozen
Wan2.1-14B~\citep{wan2025wan}, with $\tau=0$ and clean predicted and target
latent inputs.

\paragraph{Evaluation.}
Following Causal Forcing~\citep{zhu2026causal}, we evaluate 81-frame videos at
$832\times480$. VBench~\citep{huang2024vbench} Total, Quality, and Semantic
use standard prompts. Our standalone Dynamic Degree, VisionReward, and
Instruction Following scores~\citep{xu2024visionreward} use Causal Forcing's
released 100-prompt motion set and are scaled by 100. All six metrics are higher-is-better,
and negative reward scores are valid.

\subsection{Results}
\label{sec:exp_comparison}

\paragraph{Performance Comparison with Existing Models.}
Table~\ref{tab:main_comparison} compares the final generator with bidirectional,
autoregressive, and distilled autoregressive video models. We compare distilled
methods within their chunk-wise or frame-wise settings. Baseline results are
taken from the sources specified in the table caption, rather than from a
controlled retraining comparison.
RepForcing attains the highest VBench Total, Quality, VisionReward, and
Instruction Following scores among the listed distilled methods in both groups.

Under four-step chunk-wise generation, RepForcing reaches 84.66 VBench Total
and 10.074 VisionReward, exceeding the reported Causal Forcing~\citep{zhu2026causal} results by
0.62 points and 59.25\%, respectively. Relative to Causal-rCM~\citep{zheng2026causalrcm}, the corresponding
gains are 0.29 points and 22.69\%. These results show that omitting an
intermediate AR teacher and its trajectory-distillation stage need not
compromise the final model's aggregate quality or reward scores.

Under four-step frame-wise generation, RepForcing reaches 84.38 VBench Total
and 10.598 VisionReward. Compared with Causal Forcing++~\citep{zhao2026causalforcingpp}, this gives gains of
0.28 points and 55.90\%, respectively, while Instruction Following increases
from 47 to 80. The gains are not uniform across metrics: Semantic remains
below Causal Forcing~\citep{zhu2026causal} in both settings, and chunk-wise Dynamic Degree is 69
versus 71 for Causal-rCM~\citep{zheng2026causalrcm}. The results therefore support improved aggregate
quality and reward-based evaluation, rather than superiority on every metric.

\paragraph{Qualitative Comparison.}
Figure~\ref{fig:baseline} compares three sampled frames for each of four
prompts across distilled AR methods. The differences concern local detail
as well as recognizable content: RepForcing's character examples show fur
strands and facial features, while the human scenes retain clothing folds,
hanging blossoms, and rock textures. Changes in pose are visible across the
sampled frames, with scene elements such as the train interior remaining
recognizable in the RepForcing examples. These observations complement the
aggregate scores with visible differences in local texture and scene rendering.

\subsection{Ablation Studies}
\label{sec:exp_ablation}

We examine the representation model and extraction depth in the chunk-wise
four-step setting. All variants use the same training data, initial student
weights, Stage~1 update budget, cosine objective, and DMD configuration and
update budget. Table~\ref{tab:initialization} reports post-DMD results using
the metrics and evaluation protocol of Table~\ref{tab:main_comparison}.

\paragraph{Representation Model.}
\label{sec:exp_initialization}
We compare DINOv2~\citep{oquab2023dinov2}, InternVid~\citep{wang2024internvid},
and Wan2.1-14B~\citep{wan2025wan} features while keeping the DMD teacher fixed. Wan2.1-14B achieves the highest Total, Quality, Semantic,
VisionReward, and Instruction Following scores among the tested extractors.
Compared with InternVid, it improves VBench Total by 1.58 points and
VisionReward from 7.653 to 10.074. DINOv2 yields the highest Dynamic Degree,
but lower scores on the other five metrics than Wan2.1-14B. These results
support using the video diffusion representation for supervision, without
attributing the differences solely to temporal modeling.

\paragraph{Relative Extraction Depth.}
We compare features extracted at $1/8$, $1/5$, $1/2$, and $3/4$ of the
Wan2.1-14B~\citep{wan2025wan} network depth, using the selected layer's features
rather than averaging preceding layers. With 40 Transformer blocks, these
fractions correspond to blocks 5, 8, 20, and 30, respectively, counted from 1.
The $1/5$ setting achieves the highest
Total, Quality, Semantic, VisionReward, and Instruction Following scores among
the tested depths. Shallower extraction at $1/8$ yields lower scores, while
extending extraction to $1/2$ or $3/4$ provides no further improvement on these
metrics. Although $3/4$ attains the highest Dynamic Degree, its lower quality
and reward scores indicate that deeper features are not uniformly better suited
to the supervision task.

\begin{table}[t]
\centering
\small
\setlength{\tabcolsep}{4pt}
\caption{\textbf{Ablation Study on Representation Models and Extraction Depth.}
All variants are evaluated with chunk-wise four-step generation after
matched DMD refinement. Repr.\ Model compares feature extractors, while
Relative Depth varies the extraction position within Wan2.1-14B.
The default representation model is shown in bold.
Dynamic., Vision., and Instruct. denote Dynamic Degree, VisionReward,
and Instruction Following.}
\label{tab:initialization}
\resizebox{\textwidth}{!}{%
\begin{tabular}{llrrrrrr}
\toprule
Ablation & Setting & Total $\uparrow$ & Quality $\uparrow$ & Semantic $\uparrow$ & Dynamic. $\uparrow$ & Vision. $\uparrow$ & Instruct. $\uparrow$ \\
\midrule
\multirow{3}{*}{\textbf{Repr.\ Model}} & DINOv2 & 82.82 & 83.73 & 79.15 & 75 & 5.947 & 53 \\
 & InternVid & 83.08 & 84.33 & 78.05 & 64 & 7.653 & 62 \\
 & \textbf{Wan2.1-14B} & 84.66 & 85.55 & 81.10 & 69 & 10.074 & 78\\ 
\midrule
\multirow{4}{*}{\shortstack[l]{\textbf{Relative}\\\textbf{Depth}}} & $1/8$ & 84.49 & 85.45 & 80.64 & 61 & 9.217 & 69 \\
 & $\boldsymbol{1/5}$ & 84.66 & 85.55 & 81.10 & 69 & 10.074 & 78 \\
 & $1/2$ & 84.58 & 85.51 & 80.86 & 65 & 9.893 & 77 \\
 & $3/4$ & 84.44 & 85.28 & 81.09 & 71 & 9.728 & 73 \\
\bottomrule
\end{tabular}%
}
\end{table}

\section{Conclusion}
\label{sec:conclusion}

In this work, we examine the limitations of intermediate-teacher initialization
for few-step autoregressive video generation. Our analysis links latent regression
to posterior-mean predictions, while experiments reveal that causal ODE
initialization still leaves fine detail unresolved under few-step sampling.
Motivated by these findings, we propose Representation Forcing, which supervises
the causal student's clean predictions through normalized features of a frozen
video diffusion transformer. A single teacher-forced phase on the same synthetic
training videos replaces intermediate-teacher training and causal ODE distillation,
followed by asymmetric DMD. Experiments show that RepForcing better preserves
generation quality than intermediate-teacher-based initialization at four
sampling steps before DMD. After DMD, it achieves higher VBench Total and
VisionReward than the reported Causal Forcing baselines in both chunk-wise and
frame-wise generation, without additional causal ODE trajectory generation.

\bibliography{reference}
\bibliographystyle{iclr2027_conference}

\clearpage
\appendix
\section{More Visualization}
\label{sec:more_visualization}

Figure~\ref{fig:framewise_comparison} presents additional qualitative comparisons
of frame-wise autoregressive video generation. The four examples cover human
activities and animal behavior. For each prompt, we show three sampled frames
from Causal Forcing, Causal Forcing++, and RepForcing. The complete prompts are
provided in Section~\ref{sec:visualization_prompts}.

\begin{figure}[htbp]
\centering
\includegraphics[width=\linewidth]{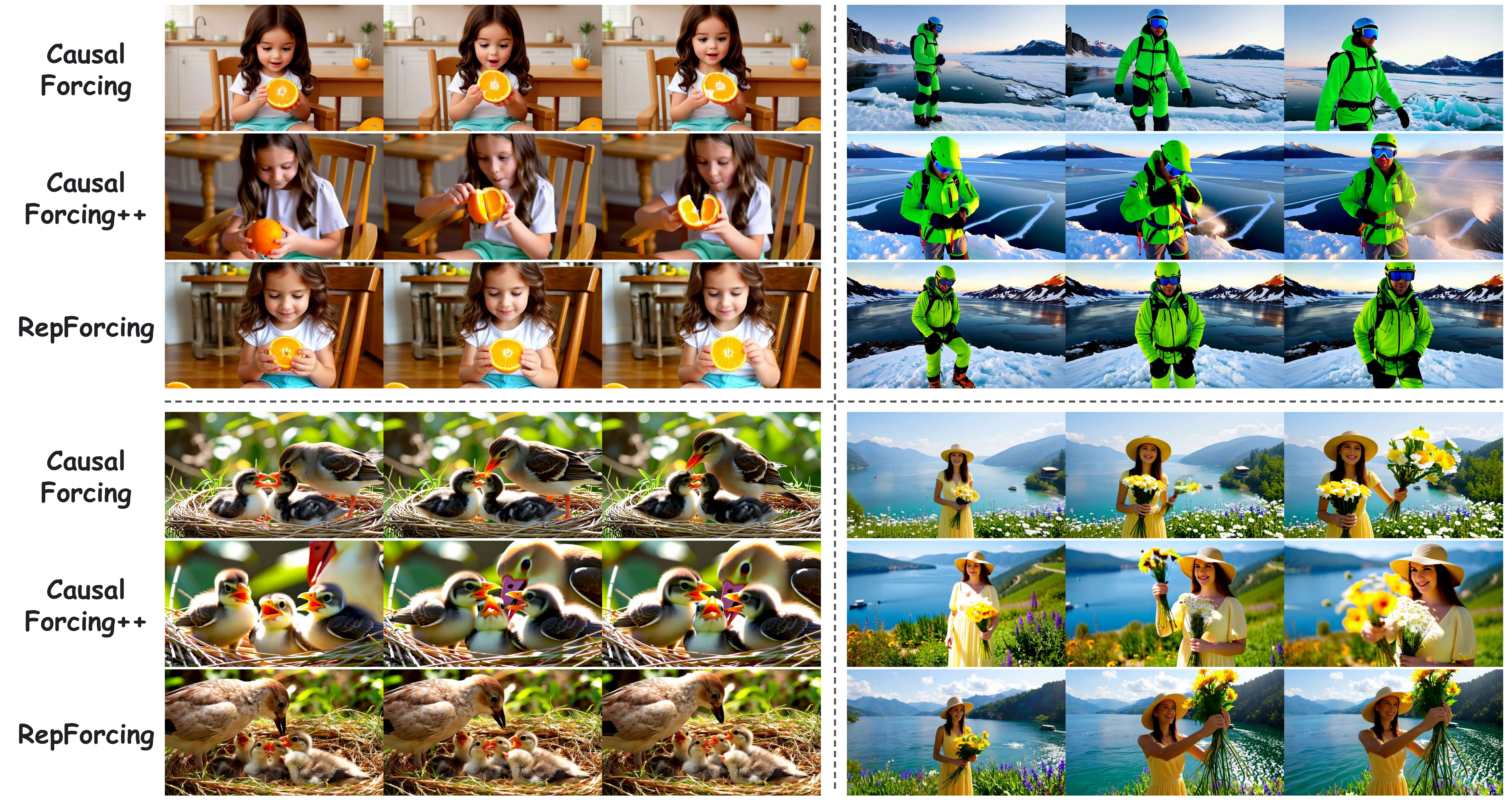}
\caption{\textbf{Qualitative comparison of frame-wise autoregressive models.}
Each quadrant shows three sampled frames for the same prompt across Causal
Forcing, Causal Forcing++, and RepForcing (top to bottom). The examples depict
a girl peeling an orange (top left), a mountaineer overlooking a frozen fjord
(top right), a mother bird feeding her chicks (bottom left), and a woman
holding flowers above a lake (bottom right).}
\label{fig:framewise_comparison}
\end{figure}

\section{Additional experimental details}
\label{sec:appendix_experiments}

\subsection{Training and evaluation protocol}
\label{sec:appendix_protocol}

RepForcing's Stage~1 reuses Causal Forcing's first-stage
$\mathcal D_{\mathrm{Bi}}$, consisting of approximately 3K clean video samples
synthesized by a bidirectional Wan model from VidProM prompts~\citep{zhu2026causal}.
These offline samples supply both the clean targets and teacher-forcing prefixes.
They are neither real captured videos nor the student's online self-rollouts
used during DMD. Causal Forcing additionally generates causal ODE trajectory
pairs with its trained AR teacher; RepForcing does not use that second dataset.
Although Causal Forcing stores bidirectional noisy intermediates when constructing
$\mathcal D_{\mathrm{Bi}}$ for an ODE baseline, RepForcing needs only the clean
video samples and independently adds noise during training.

The main comparison and representation ablations use the six metrics in
Table~\ref{tab:main_comparison}. Instruction Following is VisionReward's
prompt-alignment component, not a human-study score. Qualitative comparisons
use free autoregressive generation rather than teacher-forced reconstructions.
Dynamic Degree is interpreted alongside visual consistency because artifacts
can increase measured motion.

\subsection{Feature extractor settings}
\label{sec:appendix_ablation}

\paragraph{Default Configuration.}
For Stage~1, we extract features from the frozen Wan2.1-14B
DiT~\citep{wan2025wan} using its native bidirectional attention over the full
video. We read the output of the eighth block ($L=8$, one-based) and retain
only the first eight of its 40 blocks, corresponding to relative depth $1/5$. Both predicted and target latents are processed without VAE
decoding or added noise, with the same embedding value $\tau=0$ in the
extractor's native timestep units. This embedding is distinct from the
student's normalized noising time $t_i$. The DMD teacher instead uses the
full network to score noisy videos.

\subsection{Qualitative Comparison of Frame-Wise Generation}
\label{sec:appendix_initialization}

Figure~\ref{fig:framewise_comparison} compares Causal Forcing~\citep{zhu2026causal},
Causal Forcing++~\citep{zhao2026causalforcingpp}, and RepForcing on four prompts,
with three sampled frames per method. The examples cover a girl holding an
orange, a mountaineer beside a frozen fjord, a bird feeding its chicks, and
a woman holding flowers above a lake.

The RepForcing samples show distinct orange segments and clothing folds in
the indoor scene, together with snow texture and mountain contours in the
fjord scene. Compared with the more stylized birds in the baseline examples,
RepForcing renders fine feather texture and nest material. In the lakeside
sequence, the bouquet and surrounding landscape remain recognizable as the
woman raises the flowers. Across the displayed frames, subjects change pose
while retaining identifiable clothing, appearance, and scene elements. These
examples complement the quantitative evaluation with visual evidence of
appearance fidelity in frame-wise generation.

\section{Proofs}
\subsection{Proof of Proposition~\ref{prop:collapse}}
\label{app:proof_collapse}

The proposition concerns expected regression risk under a data distribution
with residual uncertainty in the clean continuation given the prefix and
prompt. It does not characterize memorization of a finite training set.

\begin{proof}
Fix $i$ and $t\in(0,1]$ as in the proposition. For this proof, write
$Y=\mathbf x_0^i$, $H=h_i$, and $U=(\mathbf x_t^i,H)$, and let
$m(U)=\mathbb E[Y\mid U]$. The assumed finite second moment makes the
following squared risks well defined.
For any square-integrable measurable prediction $g(U)$, conditional
expectation gives the orthogonal decomposition
\[
\mathbb E\|g(U)-Y\|_2^2
=\mathbb E\|g(U)-m(U)\|_2^2
 +\mathbb E\|Y-m(U)\|_2^2.
\]
Thus $m(U)$ is the unique minimizer up to almost-sure equality. Applying the
same orthogonality to $Y=m(U)+(Y-m(U))$ yields
\[
\mathbb E\|Y\|_2^2
=\mathbb E\|m(U)\|_2^2
 +\mathbb E\!\left[\operatorname{tr}\operatorname{Cov}(Y\mid U)\right].
\]
Here $\operatorname{Cov}$ denotes conditional covariance and
$\operatorname{tr}$ its matrix trace. For $t>0$, the conditional likelihood
of $\mathbf x_t^i$ given $Y,H$ is Gaussian with covariance $t^2 I$ and is
strictly positive. Hence the posterior law of $Y$ given $U$ is equivalent to its
conditional prior given $H$. Whenever that prior is not a point mass, the
posterior is also not a point mass, so its covariance has strictly positive
trace. This occurs on a set of positive probability by assumption. Therefore
\[
\mathbb E\|m(U)\|_2^2<\mathbb E\|Y\|_2^2.
\]
If $m(U)$ and $Y$ had the same distribution, their finite second moments
would agree, contradicting this inequality. The result concerns the direct
regression prediction at fixed $t$, not the output distribution of a complete
multi-step ODE sampler.
\end{proof}

\subsection{Conditional interpretation of representation regression}
\label{app:feature_regression}

Let $\bar\Phi(X;y)$ denote the concatenated unit-normalized token features
of the frozen video extractor defined in Sec.~\ref{sec:method_loss}.
For a generic observation $U$ containing the text condition $y$, an
unrestricted conditional predictor has an optimal video prediction
\begin{equation}
\label{eq:general_minimizer}
\mu_{\mathrm{rep}}(U)
\in\arg\min_Z\;\mathbb E\!\left[
\|\bar\Phi(Z;y)-\bar\Phi(X_0;y)\|_2^2\,\middle|\,U\right],
\end{equation}
where $Z$ ranges over candidate video latents with defined normalized
features, and a minimum is assumed to exist. Let $F=\bar\Phi(X_0;y)$ and
$m_\Phi(U)=\mathbb E[F\mid U]$. Conditional squared-error decomposition gives
\[
\begin{aligned}
\mathbb E[\|\bar\Phi(Z;y)-F\|_2^2\mid U]
&=\|\bar\Phi(Z;y)-m_\Phi(U)\|_2^2\\
&\quad+\mathbb E[\|F-m_\Phi(U)\|_2^2\mid U].
\end{aligned}
\]
The second term does not depend on $Z$. Thus an optimal prediction has
realizable features closest to $m_\Phi(U)$, which need not correspond to the
coordinate-wise mean of the target latents. The feature mean need not itself
be realizable, and the optimal video need not be unique. This characterizes
the loss preference, not the blockwise optimum of the actual causal student,
whose predictions have distinct inputs $U^i$.

\section{Prompts}
\label{sec:visualization_prompts}

We provide the complete text prompts for the qualitative examples in
Figures~\ref{teaser}, \ref{fig:nfe_sweep}, \ref{fig:baseline}, and
\ref{fig:framewise_comparison}. Within each figure, prompts are listed from
left to right and then from top to bottom across the example groups; repeated
frames and methods for the same prompt are counted only once.

\noindent The prompts used in Figure~\ref{teaser} are as follows:
\begin{promptbox}
\begin{enumerate}[label=\arabic*., leftmargin=*, itemsep=0.65\baselineskip, topsep=0pt]
\item In spring on a coastal cliff carpeted with wildflowers, a fashion model in a white blouse and sea green skirt fills the foreground in cinematic 4K. She turns into the wind, then steps forward with confident posture as the camera glides close to fluttering fabric and natural skin. Behind, turquoise waves shimmer and seabirds arc through bright air. Crisp petal texture, clean highlights, and gentle film grain make the landscape feel fresh, elegant, and hauntingly beautiful.

\item High summer on a desert highway pull‑off, a traveler in a white linen shirt and tan scarf fills the foreground in cinematic 4K. He wipes dust from sunglasses, lifts a bottle to drink, then walks toward a viewpoint as the camera tracks beside him through shimmering heat. Behind him, red mesas glow under harsh sun and a distant storm paints the sky violet. Dust motes, fabric fibers, and skin tones are crisp; gentle film grain grounds the vast, dramatic realism.

\item A cozy reading nook scene in a warm, inviting interior, featuring a playful llama sitting on a soft, plush rug. The llama is surrounded by an array of colorful, cozy pillows and soft blankets, creating a snug and comfortable atmosphere. Golden lighting from a floor lamp casts a warm glow throughout the space, enhancing the cozy ambiance. The llama reads a picture book aloud, using expressive voices to bring the characters to life. The camera captures the llama's animated face and the charming illustrations within the book, with a close-up view of both the reader and the pages.

\item In late autumn in a candle lit chapel, a violinist in a black satin dress fills the foreground in cinematic 4K. She plays a slow phrase, then accelerates into a shimmering run as the camera moves close to bow hair and fingerwork. Outside stained glass, rain taps softly and leaves swirl in the courtyard. Warm flame light paints her face and instrument, with true tones and gentle film grain making the performance hauntingly beautiful.

\item On a winter sunrise above a frozen fjord, a mountaineer in an orange down suit and mirrored goggles dominates the frame in ultra clean UHD. He tightens his harness, then steps onto a ridge as wind lifts powder into glittering sheets. The camera holds close on breath mist and frost on his lashes, then reveals blue ice, dark water cuts, and distant peaks glowing pink. Every seam, strap, and snow crystal reads sharply, with subtle film grain grounding the grandeur.

\item A vibrant and lively vlog-style photo of a corgi in tropical Maui, showcasing the dog energetically filming itself on a sandy beach. The corgi stands on the shore, one paw slightly lifted, with a joyful and curious expression. It wears a colorful collar and a small backpack camera slung over its neck. The background features a lush, palm-fringed beach with clear turquoise waters and a bright blue sky. The photo has a warm, natural lighting effect, capturing the corgi from a slightly elevated angle, emphasizing its playful and adventurous spirit.

\item In spring on a lavender hillside after rain, a woman in a cream blouse and lilac skirt dominates the frame in cinematic 4K. She lifts a bouquet to her face, then walks forward as the camera tracks close to her smile and the shimmer of dew on blossoms. Behind her, clouds break and sunlight warms purple rows stretching to a distant chapel. Sharp textures, true tones, and gentle film grain make the landscape hauntingly beautiful.

\item A serene watercolor painting of a mother duck leading her six ducklings across a tranquil pond. The mother duck has a gentle expression, her feathers glistening in the sunlight, and she frequently glances back to ensure all her ducklings are safely following in a neat little line. The ducklings follow closely behind, their small heads bobbing up and down as they waddle along. The background features a peaceful pond with lily pads and ducks floating nearby, creating a harmonious and natural scene. A mid-shot from a slightly elevated angle captures the mother duck and her ducklings in motion.
\end{enumerate}
\end{promptbox}

\noindent The prompt used in Figure~\ref{fig:nfe_sweep} is as follows:
\begin{promptbox}
\begin{enumerate}[label=\arabic*., leftmargin=*, itemsep=0.65\baselineskip, topsep=0pt]
\item Across a snowfield under pink dawn, a skier in a crimson down jacket, fur-lined hood, and mirrored goggles dominates the frame in HDR 4K with sparkling ice detail. First, she plants her poles and adjusts the strap of a small camera on her chest, breath puffing in the cold. Then she pushes off, carves a wide S-curve through powder, and throws a spray of snow toward the lens as the camera follows low and fast, mountains glowing and the sky vast and stunning.
\end{enumerate}
\end{promptbox}

\noindent The prompts used in Figure~\ref{fig:baseline} are as follows:
\begin{promptbox}
\begin{enumerate}[label=\arabic*., leftmargin=*, itemsep=0.65\baselineskip, topsep=0pt]
\item A miniature 3D render in an octane engine style depicting adorable wool and felt monsters dancing together in a dreamy, bokeh-filled setting. These soft, cuddly creatures, with big expressive eyes and fluffy bodies, are illuminated by gentle, diffused lighting that casts a warm, ethereal glow. The background features a soft, hazy backdrop with a dreamy bokeh effect, adding a cinematic quality to the scene. The monsters are shown from various angles, capturing their playful movements and expressions, creating a charming and enchanting atmosphere. A medium shot with a dynamic camera angle, highlighting the natural and joyful dance of these woolen monsters.

\item On a sunlit train crossing the Swiss Alps, a traveler in a beige coat, knitted scarf, and round sunglasses dominates the shot in 4K with crisp window reflections. First, she wipes condensation from the glass and presses her palm against it, watching snowy peaks slide past. Then she turns, opens the cabin door, and walks into the corridor with a delighted laugh, the camera tracking as valleys open outside, sunlight hits glaciers, and the landscape feels impossibly grand. with water droplets sparkling.

\item In spring under a tunnel of wisteria, a dancer in a pale lavender dress and silver hairpin fills the frame in cinematic 4K. She moves through slow turns, sleeves fluttering, then pauses to let blossoms brush her shoulders as the camera pushes close to her serene eyes. Behind her, a sunlit garden pond mirrors purple clusters and fresh green leaves. Natural tones, crisp petal detail, and soft lens bloom create a dreamy realism that feels delicate and moving.

\item At midsummer in a canyon pool, a climber in an orange technical jacket and chalk dusted hands dominates the frame in HDR UHD. He grips a hold, pulls up with visible effort, then pauses on a ledge to look down at emerald water as the camera tilts to reveal sunlight slicing through rock. Mineral streaks, wet moss, and fabric seams read sharply. True tones and subtle film grain create a thrilling sense of scale.
\end{enumerate}
\end{promptbox}

\noindent The prompts used in Figure~\ref{fig:framewise_comparison} are as follows:
\begin{promptbox}
\begin{enumerate}[label=\arabic*., leftmargin=*, itemsep=0.65\baselineskip, topsep=0pt]
\item A photo in a realistic style depicting a young girl sitting on a wooden chair, peeling an orange with a focused expression. She has long wavy brown hair and clear, warm brown eyes, wearing a simple white blouse and light blue shorts. Her hands are steady as she peels the orange, revealing the segments inside. The background shows a cozy kitchen with a blurred view of a wooden table and some utensils nearby. The lighting is soft and natural, casting gentle shadows. A close-up shot from a slightly downward angle, capturing her detailed facial expression and the orange being peeled.

\item Winter sunrise on a frozen fjord overlook, a mountaineer in a bright down suit and mirrored goggles stands close to camera in ultra‑clean 4K. He tightens his harness, stamps boots to test ice, then steps forward as wind lifts powder into glittering sheets. The camera holds on frost crystals clinging to fabric and eyelashes, then reveals blue ice, dark water cuts, and distant peaks glowing pink. Every seam and snow grain is crisp; subtle film grain grounds the grandeur.

\item A detailed and heartwarming wildlife photograph capturing a mother bird tenderly feeding her chicks in a cozy nest. The mother bird gently places food into the wide-open beaks of her chirping chicks, who eagerly await their meal. Her feathers are soft and fluffy, and she has a gentle, attentive expression. The chicks have small, round heads with wide-open beaks and big, curious eyes. The nest is lined with soft grass and twigs, and the background features a blurred forest scene with dappled sunlight filtering through the leaves. The photo has a natural, documentary style. A close-up shot from a slightly elevated angle, focusing on the interaction between the mother bird and her chicks.

\item In spring on a hillside above a bright lake, a woman in a light yellow dress and straw hat dominates the frame in cinematic 4K. She picks wildflowers, then lifts the bouquet toward the lens as the camera glides closer to her joyful expression. Behind her, water shimmers, small boats drift, and mountains soften into blue haze. Sharp textures, true tones, and gentle film grain create a warm, uplifting beauty.
\end{enumerate}
\end{promptbox}

\end{document}